\documentclass[twocolumn]{article}
\usepackage{latexsym,amssymb,amsmath,my,graphicx}

\newtheorem{lemma}{Lemma}[section]

\newtheorem{definition}{Definition}[section]

\def\tr{^\top}
\def\Re{\mathbb R}

\def\sp{\mathop{\text{\rm span}}}
\def\nullsp{\mathop{\text{\rm null}}}

\title{\rule[.5ex]{\textwidth}{.5pt}
Galerkin Methods for Complementarity Problems and\\ Variational Inequalities
 \rule[.8ex]{\textwidth}{.5pt}}
\author{Geoffrey J. Gordon \\ Machine Learning Department \\
  Carnegie Mellon University \\ {\tt ggordon@cs.cmu.edu}}

\begin{document}

\maketitle

\begin{abstract}
  Complementarity problems and variational inequalities arise in a
  wide variety of areas, including machine learning, planning, game
  theory, and physical simulation.  In all of these areas, to handle
  large-scale problem instances, we need fast approximate solution
  methods.  One promising idea is Galerkin approximation, in which we
  search for the best answer within the span of a given set of basis
  functions.  Bertsekas~\cite{bertsekas-galerkin} proposed one
  possible Galerkin method for variational inequalities.  However,
  this method can exhibit two problems in practice: its approximation
  error is worse than might be expected based on the ability of the
  basis to represent the desired solution, and each iteration requires
  a projection step that is not always easy to implement efficiently.
  So, in this paper, we present a new Galerkin method with improved
  behavior: our new error bounds depend directly on the distance from
  the true solution to the subspace spanned by our basis, and the only
  projections we require are onto the feasible region or onto the span
  of our basis.
\end{abstract}

\section{Background}

\subsection{Definitions}

We first define variational inequalities and complementarity
problems~\cite{facchinei-pang-vi-cp}.  We consider here only problems
defined over convex feasible sets, although it is possible to define
analogous problems over nonconvex sets.  

A cone is a set $K$ such that, for any $\lambda\geq 0$ and $x\in K$,
we have $\lambda x\in K$.  The dual cone of $K$ is
\[
K^* = \{ y\mid \forall x\in K,\, x\tr y\geq 0 \}
\]
For any set $C$ and point $x\in C$, the normal cone to $C$ at $x$ is
$N_C(x)=\{d\mid d\tr (y-x)\leq 0,\,\forall y\in C\}$.

\begin{definition}[Variational inequality] Given a nonempty closed
  convex set $C$ and an operator $F$, the variational inequality
  VI$(F,C)$ is to find $x$ s.t.:
  \begin{align*}
    x &\in C\\
    -F(x) &\in N_C(x)
  \end{align*}
  The VI is linear if $F(x)=Mx+q$.
\end{definition}

In a complementarity problem, the feasible set $K$ is required to be a
cone.  Define $(x,y)$ to be $K$-complementary if $x\in K$, $y\in K^*$,
and $x\tr y=0$.  Then the complementarity problem asks us to find a
$K$-complementary pair $(x,F(x))$:

\begin{definition}[Complementarity problem] Given a nonempty closed
  convex cone $K$ and an operator $F$, the complementarity problem
  CP$(F,K)$ is to find $x$ s.t.:
\begin{align*}
x &\in K\\
F(x) &\in K^*\\
x\tr F(x) &= 0
\end{align*}
The CP is linear if $F(x)=Mx+q$.
\end{definition}

A common choice is to take $K$ to be the
nonnegative orthant; in this case complementarity means that, for all
$i$, at least one of $x_i$ and $[F(x)]_i$ is zero.

\subsection{Relationships}

Variational inequalities and complementarity problems are strongly
related.  First, if the feasible set of a variational inequality is a
cone $K$, then VI$(F,K)$ is equivalent to CP$(F,K)$; that is, both
problems have the same set of solutions.  Second, if the feasible set
is the intersection of a cone $K$ with some equality constraints
$Ax=b$, then we can eliminate the equality constraints using Lagrange
multipliers: VI$(F,K\cap\{x\mid Ax=b\})$ is equivalent to VI$(\bar F,
\bar K)$, where $\bar K = K\times \Re^m$ and
\begin{align*}
  \bar F(x,\lambda) = \left(
    \begin{array}{c}
      F(x) - A\tr\lambda\\
      Ax-b
    \end{array}
  \right)
\end{align*}
Since we can represent any convex set as the intersection of a cone
with equality constraints, the above relationships mean that we can
transform any VI to a CP and vice versa.

For computational purposes, it is often convenient to transform our
problem so that the feasible set is a very simple cone.  For example,
for the polyhedron $\{x \mid Ax + b \geq 0\}$, we introduce a vector
of nonnegative slack variables $s$ and write $Ax+b=s$, with $x$ free
and $s \geq 0$.  After eliminating the equality constraints with
Lagrange multipliers, our feasible set is the cone $K = \{(s, x,
\lambda)\mid s\geq 0\}$.  This cone is \emph{separable}: it is the
product of one-dimensional cones.  The advantage of this sort of
transformation is that it is extremely efficient to work with
separable cones: e.g., Euclidean projection onto $K$ just means
thresholding each component of $s$ at $0$.

\subsection{Complexity}

If we assume that the feasible set $K$ is a separable cone, then the
computational complexity of a variational inequality or
complementarity problem depends on the operator $F$.  Even if $F$ is
restricted to be linear, it is possible to encode NP-hard problems.
An important class that ensures polynomial-time solvability is the
class of \emph{monotone} operators:

\begin{definition}
  An operator $F$ is monotone on the set $C$ if, for some $\beta\geq
  0$,
  \begin{align*}
    (x-y)\tr(F(x)-F(y)) &\geq \beta \|x-y\|^2\quad\forall x, y\in C
  \end{align*}
  It is strongly monotone if the above holds with $\beta > 0$.
\end{definition}

A linear operator $F(x)=Mx+q$ is 
monotone iff $M$ is
positive semidefinite, and strongly monotone iff $M$ is positive
definite.  (There is no requirement for $M$ to be symmetric.)

Among Lipschitz operators (which include all finite-dimensional linear
operators), strong monotonicity implies a useful contraction property:
\begin{lemma}
  \label{lem:pd-contract}
  If the operator $F$ is $\beta$-strongly monotone and $L$-Lipschitz
  on some set $C$, and if we take $\alpha = \beta/L^2$, then the
  operator $I-\alpha F$ is Lipschitz on $C$ with constant
  $\gamma=\sqrt{1-\beta^2/L^2}<1$.
\end{lemma}
\begin{proof}
  For any $x,y\in C$,
  \begin{align*}
    \lefteqn{\|(x-\alpha F(x))-(y-\alpha F(y))\|^2}\quad &\\
    &= \|(x-y)-\alpha (F(x)-F(y))\|^2\\
    &= \|x-y\|^2-2\alpha(x-y)\tr(F(x)-F(y))\\
    &\qquad{}+\alpha^2\|F(x)-F(y)\|^2\\
    &\leq \|x-y\|^2-2\alpha\beta\|x-y\|^2+\alpha^2L^2\|x-y\|^2\\
    &= (1-2\alpha\beta+\alpha^2L^2)\|x-y\|^2
  \end{align*}
  Choosing $\alpha$ to minimize the RHS, we have $\alpha = \beta/L^2$;
  the square of the Lipschitz constant is then
  \[
  1-2\alpha\beta+\alpha^2L^2 = 1 - \beta^2/L^2 < 1
  \]
  as claimed.  %
\end{proof}

This contraction property means that a very simple algorithm, the
\emph{projection method}, converges linearly to the solution of our
variational inequality, as we will see in the next section.

\section{Projection method}

In a variational inequality,
the condition that $-F(x)$ is in the normal cone to $C$ at $x$ is
equivalent to
\begin{align} 
  \label{eq:proj-condition}
  x &= \Pi_C(x-\alpha F(x)) 
\end{align}
where $\alpha$ is any nonnegative step size and $\Pi_C$ represents
Euclidean projection onto the set $C$.  That is, if we take a step in
the direction $-F(x)$ and project back onto $C$, we don't move.  More
formally:
\begin{lemma}
  \label{lem:normalcone}
  If $x = \Pi_C(x+d)$ for some step $d$, then $d\in N_C(x)$.
\end{lemma}
(We omit the proof.)

The projection method simply treats~(\ref{eq:proj-condition}) as an
assignment: start from an arbitrary $x^{(0)}$, then for
$t=1,2,\ldots$, set 
\begin{align}
  \label{eq:proj-method}
  x^{(t)}&=T(x^{(t-1)})\equiv\Pi_C(x^{(t-1)}-\alpha F(x^{(t-1)}))
\end{align}
(We could also choose a separate step size $\alpha_t$ for each step
$t$, but for simplicity we ignore this possibility.)

The projection method converges linearly for strongly-monotone $F$,
since its update operator is a contraction:
\begin{lemma}
  \label{lem:proj}
  In the projection method, suppose the operator $(I-\alpha F)$ is
  $\gamma$-Lipschitz.  Then the variational
  inequality~(\ref{eq:proj-condition}) has a unique solution $x^*$.
  Furthermore, for any $\epsilon>0$, if we take $t\geq
  \ln(\epsilon)/\ln(\gamma)$ steps of the projection method, we will
  have $\|x^{(t)}-x^*\|\leq\epsilon\|x^{(0)}-x^*\|$.
\end{lemma}
\begin{proof}
  The update operator $T=\Pi_C\circ(I-\alpha F)$ is
  $\gamma$-Lipschitz, since it is the composition of a projection
  operator (which is $1$-Lipschitz) with $(I-\alpha F)$.  Existence
  and uniqueness of the solution follow from the contraction mapping
  theorem; for the error bound,
  \begin{align*}
    \|x^{(t)}-x^*\| &= \|T x^{(t-1)} - Tx^*\|\\
    &\leq \gamma \|x^{(t-1)}-x^*\|\\
    &\leq \gamma^t \|x^{(0)}-x^*\|
  \end{align*}
  That is, the distance between $x^{(t)}$ and $x^*$ reduces by a
  factor $\gamma$ with each application of $T$.  Substituting the
  assumed value for $t$ now yields the desired bound.
\end{proof}

\section{Galerkin method---I}

If the dimension of $x$ is very large, each iteration of the
projection method may be too slow to be practical.  The problem is
exacerbated if $\gamma$ is close to 1.  
In this case, it makes sense to search for a Galerkin approximation,
i.e., an approximate solution within the span of some basis matrix
$\Phi$.  The hope is that we can design an approximate projection
method that will find such a solution much more cheaply than the raw
projection method.

Bertsekas~\cite{bertsekas-galerkin} proposed one such Galerkin
approximation: assume that $\hat C \equiv \sp(\Phi)\cap C$ is
nonempty, and define the approximation as the fixed point of
\begin{align}
  \label{eq:berts-proj-method}
  x^{(t)}&=\hat T(x^{(t-1)})\equiv\Pi_{\hat C}(x^{(t-1)}-\alpha F(x^{(t-1)}))
\end{align}
Compared to~(\ref{eq:proj-method}), we simply project onto $\hat C$
instead of $C$ at each iteration.  Note
that~(\ref{eq:berts-proj-method}) also defines an algorithm for
computing the Galerkin approximation, although not the only possible
one.

If the dimension of $x$ is $n$, but the rank of $\Phi$ is $k<n$, we
may be able to implement the iteration~(\ref{eq:berts-proj-method})
cheaply: e.g., for a linear variational inequality, we can precompute
$\Phi\tr M \Phi$ and $\Phi\tr q$, and work in $k$-dimensional space
instead of $n$-dimensional space.  The only caveat is projection onto
$\hat C$; depending on the forms of $\Phi$ and $C$, this projection
may or may not be possible to implement efficiently.

Essentially the same convergence rate analysis holds for the
Bertsekas-Galerkin projection method as for the original projection
method:

\begin{lemma}
  Suppose  $(I-\alpha F)$ is $\gamma$-Lipschitz and  $\hat C$
  is nonempty.
  Then the
  iteration~(\ref{eq:berts-proj-method}) has a unique fixed point, say
  $\hat x$; and, we need at most $t\geq \ln(\epsilon)/\ln(\gamma)$
  iterations to achieve error $\|x^{(t)}-\hat
  x\|\leq\epsilon\|x^{(0)}-\hat x\|$.
\end{lemma}
\begin{proof}
  Identical to the proof of Lemma~\ref{lem:proj}.
\end{proof}

In addition, we can prove a bound on the distance between the
approximate solution $\hat x$ and the original solution $x^*$:
\begin{lemma}
  \label{lem:berts-err-bnd}
  For the Bertsekas-Galerkin method, the error between the approximate
  solution $\hat x$ and the true solution $x^*$ satisfies:
  \begin{align}
    \label{eq:berts-err-bnd}
    \|\hat x-x^*\| &\leq \|\Pi_{\hat C}(x^*)-x^*\|/(1-\gamma)
  \end{align}
\end{lemma}

\begin{proof}
  Note that $\hat T(x) = \Pi_{\hat C}(T(x))$, and $T(x^*)=x^*$.  So,
  $\hat T(x^*)=\Pi_{\hat C}(T(x^*))=\Pi_{\hat C}(x^*)$, and
  \begin{align*}
    \|\hat T(x^*)-x^*\|
    &= \|\Pi_{\hat C}(x^*)-x^*\|
  \end{align*}
  Applying $\hat T$ repeatedly to both $x^*$ and $\hat T(x^*)$, and
  using the fact that $\hat T$ is $\gamma$-Lipschitz, we have
  \begin{align*}
    \|\hat T(\hat T(x^*))-\hat T(x^*)\|
    &\leq \gamma\|\Pi_{\hat C}(x^*)-x^*\|\\
    \|\hat T^{t}(x^*)-\hat T^{t-1}(x^*)\|
    &\leq \gamma^{t-1}\|\Pi_{\hat C}(x^*)-x^*\|
  \end{align*}
  So, after $t$ iterations, we have
  \begin{align*}
    \|\hat T^{t}(x^*)-x^*\|
    &\leq \sum_{i=1}^t\|\hat T^{i}(x^*)-\hat T^{i-1}(x^*)\|\\
    &\leq \sum_{i=1}^t\gamma^{i-1}\|\Pi_{\hat C}(x^*)-x^*\|\\
    &\leq \|\Pi_{\hat C}(x^*)-x^*\|/(1-\gamma)
  \end{align*}
  where the last line follows by summing a geometric series.
  Since this inequality holds for all $t>0$, it holds in the limit as
  $t\to\infty$, yielding~(\ref{eq:berts-err-bnd}).
\end{proof}

The bound~(\ref{eq:berts-err-bnd}) depends on $\|\Pi_{\hat
  C}(x^*)-x^*\|$, a measure of our \emph{representation error}---that
is, how closely we can approximate $x^*$ in our low-dimensional
representation.

Qualitatively, compared to the exact projection method, the
Bertsekas-Galerkin method tightens the feasible region and relaxes the
optimality condition: both the iterates $x^{(t)}$ and the approximate
solution $\hat x$ will be exactly feasible.  (In fact, they will be in
$\hat C\subseteq C$.)  But, at convergence, the step direction
$-F(\hat x)$ will not necessarily be in a normal cone of $C$, as it
would in the exact solution.  Instead, Lemma~\ref{lem:normalcone}
means that the step direction will be in the normal cone to $\hat C$
at $\hat x$.

Another way to interpret this approximate optimality criterion is
that, since $\hat C=C\cap\sp(\Phi)$ is the interection of two sets,
the normal cone to $\hat C$ is the sum of their two normal cones,
$N_C(\hat x)+\nullsp(\Phi\tr)$.  In other words, for some residual
vector $\epsilon\in\nullsp(\Phi\tr)$, we have that $-F(\hat
x)+\epsilon$ is in the normal cone to $C$ at $\hat x$.

\section{Galerkin method---II}

While the error bound~(\ref{eq:berts-err-bnd}) is nice to have, it is
less than perfectly satisfying: the representation error measure
$\|\Pi_{\hat C}(x^*)-x^*\|$ may be quite large, even if $\sp(\Phi)$
passes close to $x^*$.  See Fig.~\ref{fig:proj-error} for an example.
An
extreme version of this problem is that $\sp(\Phi)$ might not even
intersect the feasible set $C$.

\begin{figure}
\centerline{\includegraphics[width=0.5\columnwidth]{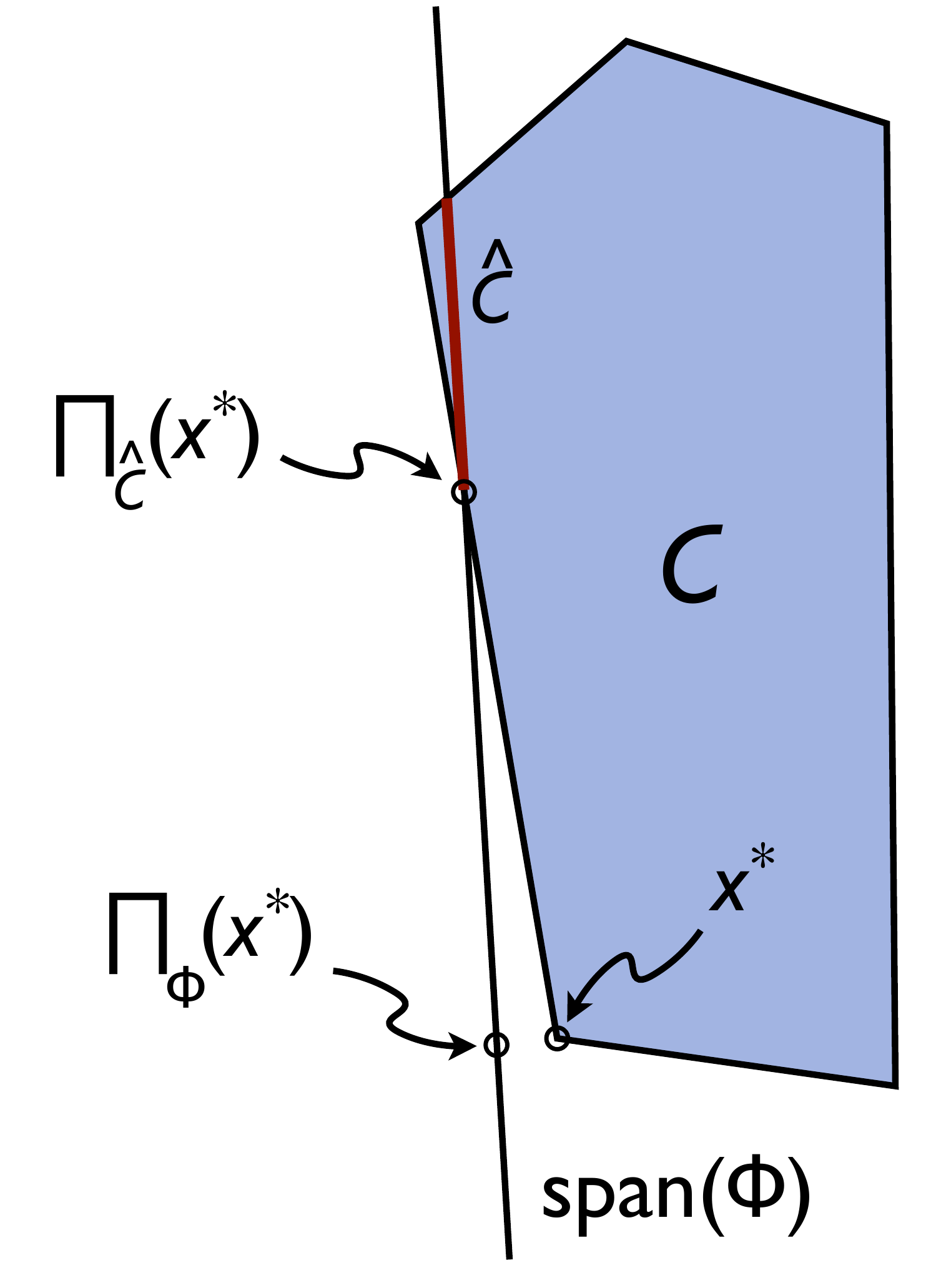}}
\caption{Illustration of error bound~(\ref{eq:berts-err-bnd}).}
\label{fig:proj-error}
\end{figure}

To remedy this problem, in this section we propose a simple new
Galerkin method, and bound the new method's convergence rate and
error.  Our new method's bounds will depend on a different measure of
representation error, based on projection onto $\sp(\Phi)$ instead of
$\hat C$.  Since $\sp(\Phi)\supseteq\hat C$, we expect this change to
lead to an improvement in the bound: for any given vector $z$, the
error $\|\Pi_\Phi(z)-z\|$ will be no larger, and sometimes much
smaller, than the error $\|\Pi_{\hat C}(z)-z\|$.

To enable projection onto $\sp(\Phi)$ instead of $\hat C$, we modify
the algorithm as described below; in particular, we project the
proposed update $z = x-\alpha F(x)$ instead of the proposed solution
$x$.  When $\alpha$ is small, we have $z\approx x$, so that
$\|z-\Pi_\Phi(z)\| \lesssim \|x-\Pi_{\hat C}(x)\|$.  For larger
$\alpha$ it is possible that this inequality could be reversed; but in
our (limited) experience, even for larger $\alpha$, the new method's
bounds are at least as tight as those of the Bertsekas-Galerkin
method, and often much tighter.

In more detail, our approximation is based on the following
re-arrangement of the projection method: pick $z^{(0)}$ or $x^{(0)}$
arbitrarily, and for all $t=1,2,\ldots$, repeat:
\begin{align}
  \label{eq:new-proj-method1}
  x^{(t-1)} &= \Pi_C(z^{(t-1)})\\
  \label{eq:new-proj-method2}
  z^{(t)} &= x^{(t-1)} - \alpha F(x^{(t-1)})
\end{align}
It is easy to see that the sequence $x^{(t)}$ computed by this
iteration is the same as the one computed by~(\ref{eq:proj-method}),
so long as we start from the same value of $x^{(0)}$.  

Now we insert a projection operator into~(\ref{eq:new-proj-method2}),
and define our new Galerkin approximation as the fixed point of the
resulting iteration:
\begin{align}
  \label{eq:new-gal-method1}
  x^{(t-1)} &= \Pi_C(z^{(t-1)})\\
  \label{eq:new-gal-method2}
  z^{(t)} &= \Pi_\Phi(x^{(t-1)} - \alpha F(x^{(t-1)}))
\end{align}
(As before, fixed-point iteration is only one way to compute the
approximation; we will discuss others below.)

Right away we can see a difference between the new method and the old
one: instead of projecting onto the intersection of $C$ and
$\sp(\Phi)$, we project onto each of $C$ and $\sp(\Phi)$ at different
points in the update.  So, for example, it is not even necessary that
$C\cap\sp(\Phi)$ be nonempty.

Another important difference is what is required for efficient
implementation of the
iteration~(\ref{eq:new-gal-method1}--\ref{eq:new-gal-method2}).  For
example, suppose that $C$ is the nonnegative orthant and $\Phi$ is an
arbitrary $n\times k$ basis matrix.  The Bertsekas-Galerkin method
requires us to project onto $C\cap\sp(\Phi)$, a quadratic program in
$n$ dimensions.  The new method requires only separate projections
onto $C$ and $\sp(\Phi)$, both of which are much faster than solving a
quadratic program: the former is componentwise thresholding, while the
latter is a rank-$k$ linear operator which we can precompute.  (For
example, if $\Phi$ is dense, a good strategy might be to precompute
its QR decomposition; then we can threshold and project in total time
$O(nk)$ with a very low constant.)

As before, the entire update is $\gamma$-Lipschitz (being the
composition of $(I-\alpha F)$ with some projections, each of which is
$1$-Lipschitz).  We therefore have:
\begin{lemma}
  If $(I-\alpha F)$ is $\gamma$-Lipschitz, then the
  update~(\ref{eq:new-gal-method1}--\ref{eq:new-gal-method2}), viewed
  either as an update for $x$ or as an update for $z$, is
  $\gamma$-Lipschitz.  So, the
  iteration~(\ref{eq:new-gal-method1}--\ref{eq:new-gal-method2}) has a
  unique fixed point, say $(\bar x,\bar z)$; and, we need at most
  $t\geq \ln(\epsilon)/\ln(\gamma)$ iterations to achieve error
  $\|x^{(t)}-\bar x\|\leq\epsilon\|x^{(0)}-\bar x\|$ or
  $\|z^{(t)}-\bar z\|\leq\epsilon\|z^{(0)}-\bar z\|$.
\end{lemma}
\begin{proof}
  Identical to the proof of Lemma~\ref{lem:proj}.
\end{proof}

Write $x^*$ and $z^*$ for the true solution to our variational
inequality, that is, the fixed point
of~(\ref{eq:new-proj-method1}--\ref{eq:new-proj-method2}).  We can
bound the distance between $\bar z$ and $z^*$:

\begin{lemma}
  For the Galerkin method
  of~(\ref{eq:new-gal-method1}--\ref{eq:new-gal-method2}), the error
  between the approximate solution $\bar z$ and the true solution
  $z^*$ satisfies:
  \[
  \|z^*-\bar z\| \leq \|z^* - \Pi_\Phi z^*\|/(1-\gamma)
  \]
  Similarly, the error in $\bar x$ satisfies
  \[
  \|x^*-\bar x\| \leq \|z^* - \Pi_\Phi z^*\|/(1-\gamma)
  \]
\end{lemma}

\begin{proof}
  We can follow almost the same proof strategy as in
  Lemma~\ref{lem:berts-err-bnd}: the update operator for $z$,
  \[
  \Pi_\Phi \circ (I-\alpha F) \circ \Pi_C
  \]
  is $\gamma$-Lipschitz as discussed above, so to get our error bound
  we just need to bound the change in a single update starting from
  $z^*$.  Since $z^*=(I-\alpha F)\circ\Pi_C(z^*)$,
  \begin{align*}
    \|z^* - \Pi_\Phi \circ (I-\alpha F) \circ \Pi_C (z^*)\|
    &= \|z^* - \Pi_\Phi z^*\|
  \end{align*}
  Just as before, we can now apply the update operator repeatedly to
  get bounds on the difference between successive iterates, and sum
  these bounds to show the desired result.
  
  To bound the error in $\bar x$, note  $\bar x = \Pi_C(\bar z)$
  and $x^* = \Pi_C(z^*)$, and use the fact that $\Pi_C$ is
  $1$-Lipschitz.
\end{proof}

As promised, comparing to~(\ref{eq:berts-err-bnd}), the form of the
new bound is identical except that we measure representation error as
$\|z^*-\Pi_\Phi(z^*)\|$ rather than $\|x^*-\Pi_{\hat C}(x^*)\|$.

Qualitatively, both the iterates $x^{(t)}$ and the approximate
solution $\bar x$ will be exactly feasible (i.e., $x^{(t)}\in C$,
$\bar x\in C$), just as for the Bertsekas-Galerkin method.  The
approximate optimality criterion is also similar: for some residual
vector $\epsilon\in\nullsp(\Phi\tr)$, we have that $-F(\bar
x)+\epsilon$ is in the normal cone to $C$ at $\bar x$ (see
Lemma~\ref{lem:new-gal-converge}).  So, the most important difference
between the two methods is that the new method searches over a larger
feasible region ($C$ instead of $\hat C$).

\begin{lemma}\label{lem:new-gal-converge}
  At the fixed point $(\bar x, \bar z)$
  of~(\ref{eq:new-gal-method1}--\ref{eq:new-gal-method2}), we have
  $-F(\bar x)+\epsilon\in N_C(\bar x)$ for some residual vector
  $\epsilon\in\nullsp(\Phi\tr)$.
\end{lemma}

\begin{proof}
  By the assumption that $(\bar x, \bar z)$ is a fixed point, we have
  \begin{align}
    \label{lemngc:1}
    \bar x &= \Pi_C(\bar z)\\
    \label{lemngc:2}
    \bar z &=\Pi_\Phi(I-\alpha F)(\bar x)
  \end{align}
  (from~(\ref{eq:new-gal-method1}) and~(\ref{eq:new-gal-method2})
  respectively).  If we choose an appropriate
  $\epsilon'\in\nullsp(\Phi\tr)$, we can eliminate the projection
  operator $\Pi_\Phi$ from~(\ref{lemngc:2}), getting:
  \begin{align*}
    \bar z &= (I-\alpha F)(\bar x) + \epsilon' \\
    \bar z - \bar x &= -\alpha F(\bar x) + \epsilon' 
  \end{align*}
  Now $\bar z-\bar x\in N_C(\bar x)$, by~(\ref{lemngc:1}) and
  Lemma~\ref{lem:normalcone}.  So, taking $\epsilon=\epsilon'/\alpha$,
  the desired result follows.
\end{proof}

\section{Projective LCPs}

The projection method described above is by no means the only way to
find the Galerkin
approximation~(\ref{eq:new-gal-method1}--\ref{eq:new-gal-method2}).
If the operator $F$ is linear, another attractive algorithm is the
following interior point method.

Let $F(x)=Mx+q$, and suppose that the feasible set is a separable cone
$K$, so that we are solving the two equivalent problems VI$(F,K)$ and
CP$(F,K)$.  Saying that $x$ is a fixed point
of~(\ref{eq:new-gal-method1}--\ref{eq:new-gal-method2}) is equivalent
to
\begin{align}
x %
&= \Pi_K(\Pi_\Phi((I-\alpha M) x-\alpha q)) \nonumber\\
&= \Pi_K((I-(I-\Pi_\Phi+\alpha\Pi_\Phi M)) x-\alpha \Pi_\Phi q) \nonumber\\
&= \Pi_K((I-N)x-r) \label{eq:n-lcp}
\end{align}
where we have defined 
\begin{align}
\label{eq:n-def}
N=I-\Pi_\Phi+\alpha\Pi_\Phi M \qquad r=\alpha \Pi_\Phi q
\end{align}
(The first line substitutes $F(x)=Mx+q$
into~(\ref{eq:new-gal-method1}--\ref{eq:new-gal-method2}); the second
line distributes $\Pi_\Phi$ and then adds and subtracts $I$; the third
substitutes the definitions of $N$ and $r$.)

But now we can recognize~(\ref{eq:n-lcp}) as the optimality condition for
a new problem, CP$(Nx+r, K)$.  So, we can use any appropriate linear
complementarity problem solver to solve CP$(Nx+r, K)$.  Furthermore,
if the original matrix problem was strongly monotone, then the new
problem is strongly monotone too:

\begin{lemma}
  If $M$ is positive definite and $\alpha$ is chosen as in
  Lemma~\ref{lem:pd-contract}, then $N$ as defined in~(\ref{eq:n-def})
  is positive definite.
\end{lemma}
\begin{proof}
  First note that $I-N$ is $\gamma$-Lipschitz, where $\gamma\in[0,1)$ is
  defined in Lemma~\ref{lem:pd-contract}: $I-N=\Pi_\Phi(I-\alpha M)$,
  and the RHS is the composition of a projection (which is
  $1$-Lipschitz) with $(I-\alpha M)$ (which is $\gamma$-Lipschitz).
  Now pick any $x\neq 0$.  We will show $x\tr N x > 0$, which implies
  that $N$ is positive definite:
  \begin{align*}
    \|(I-N)x\|^2 &\leq \gamma^2\|x\|^2\\
    x\tr x - 2x\tr N x + x\tr N\tr N x &\leq \gamma^2 x\tr x\\
    (1-\gamma^2) x\tr x + x\tr N\tr N x &\leq 2 x\tr N x\\
    0 &< x\tr N x
  \end{align*}
  The first line holds since $(I-N)$ is $\gamma$-Lipschitz.  The
  second line expands the squares.  The third line rearranges and
  collects terms.  The last line follows since the LHS of line 3 is
  strictly positive: $(1-\gamma^2)x\tr x$ is strictly positive since
  $x\neq 0$ and $\gamma < 1$, and $x\tr N\tr N x=\|Nx\|^2\geq 0$.
\end{proof}

The dimension of the new matrix $N$ is the same as that of the
original matrix $M$, so it is not clear that we are making progress by
expressing our LCP this way.  However, it turns out that we can solve
CP$(Nx+r,K)$ much more quickly than a general LCP of the same
dimension, due to the special structure of $N$.

In particular, $N$ is \emph{projective} in the sense
of~\cite{projective-lcp}.  So, as shown in that paper, we can run the
Unified Interior Point (UIP) method of Kojima et
al.~\cite{kojima-etal-uip} very quickly: if $M\in\Re^{n\times n}$ and
$\Phi\in\Re^{n\times k}$, then each iteration of the UIP method takes
time $O(nk^2)$, only linear in $n$.  (By contrast, an iteration of UIP
ordinarily requires solving an $n\times n$ system of equations, usually
much more expensive if $k\ll n$.)  Furthermore, as shown
in~\cite{kojima-etal-uip}, the number of iterations required is
polynomial for any monotone LCP\@.

\bibliographystyle{unsrt}
\bibliography{lcpmdp}

\end{document}